\documentclass[11pt]{article} 
\usepackage{times}

\usepackage{graphicx} 
\usepackage{amsfonts}
\usepackage{amsmath}
\usepackage{graphicx}

\usepackage{ifthen}
\newboolean{SHORTVERSION}
\setboolean{SHORTVERSION}{true}

\usepackage[ruled,algo2e]{algorithm2e}

\newtheorem{remark}{Remark}

\newtheorem{theorem}{Theorem}

\newtheorem{lemma}[theorem]{Lemma}

\newcommand{\qed}{\hfill $\Box$}
\newenvironment{proof}{\par\noindent{\bf Proof.}}{\qed \par\smallskip\noindent}
\newenvironment{proofof}[1]{\par\noindent{\bf Proof of #1.}}{\qed \par\smallskip\noindent}

\newcommand{\theset}[2]{ \left\{ {#1} \,:\, {#2} \right\} }
\newcommand{\Ind}[1]{ \mathbb{I}\left\{{#1}\right\} }
\newcommand{\alg}{\textsc{Shazoo}}
\newcommand{\treeopt}{\textsc{TreeOpt}}
\newcommand{\wta}{\textsc{wta}}
\newcommand{\omv}{\textsc{omv}}
\newcommand{\labprop}{\textsc{LabProp}}
\newcommand{\sgn}{\mathrm{sgn}}
\newcommand{\scO}{\mathcal{O}}
\newcommand{\scC}{\mathcal{C}}
\newcommand{\yhat}{\widehat{y}}
\newcommand{\by}{\boldsymbol{y}}
\newcommand{\cf}{M_C^{F}}
\newcommand{\ci}{M_C^{\mathrm{in}}}
\newcommand{\co}{M_C^{\mathrm{out}}}
\newcommand{\cF}{C^{F}}
\newcommand{\cut}{\mathrm{cut}}
\newcommand{\fcut}{\mathrm{fcut}}
\newcommand{\spin}{\{-1,+1\}}

\title{
See the Tree Through the Lines: The Shazoo Algorithm\\
-- Full Version --
}

\author{
Fabio Vitale\\
DSI, University of Milan, Italy\\
{\tt fabio.vitale@unimi.it}
\and
Nicol\`o Cesa-Bianchi\\ 
Dipartimento di Informatica, Universit\`a degli Studi di Milano, Italy\\
\texttt{nicolo.cesa-bianchi@unimi.it}
\and
Claudio Gentile\\ 
DiSTA, Universit\`a dell'Insubria, Italy\\
\texttt{claudio.gentile@uninsubria.it}
\and
Giovanni Zappella\\
Dipartimento di Matematica, Universit\`a degli Studi di Milano, Italy\\
\texttt{giovanni.zappella@unimi.it}
}

\newcommand{\mycom}[1]{\textcolor{red}{[NCB: {#1}]}}

\begin{document}

\maketitle

\begin{abstract}
Predicting the nodes of a given graph is a fascinating 
theoretical problem with applications in several domains. 
Since graph sparsification via spanning trees 
retains enough information while making the task much easier, 
trees are an important special case of this problem.
Although it is known how to predict the nodes of an unweighted tree 
in a nearly optimal way, in the weighted case a fully satisfactory 
algorithm is not available yet. We fill this hole and introduce an efficient node predictor, 
\alg, which is nearly optimal on any weighted tree. Moreover, we show that \alg\ can 
be viewed as a common nontrivial generalization of both previous approaches for 
unweighted trees and weighted lines. 
Experiments on real-world datasets confirm that \alg\ performs well in that
it fully exploits the structure of the input tree,
and gets very close to (and sometimes better than)
less scalable energy minimization methods.
\end{abstract}

\section{Introduction}
%
Predictive analysis of networked data is a fast-growing research area whose application domains 
include document networks, online social networks, and biological networks. In this work
we view networked data as weighted graphs, and focus on the task of node classification in the 
transductive setting, i.e., when the unlabeled graph is available beforehand. 
Standard transductive classification methods, such as label propagation~\cite{BMN04,BDL06,ZGL03}, 
work by optimizing a cost or energy function defined on the graph, which includes the training information as labels assigned to training nodes. 
Although these methods perform well in practice, they are often computationally expensive, and have performance guarantees that require statistical assumptions on the selection of the training nodes.

A general approach to sidestep the above computational issues is to sparsify the graph to the largest possible extent,
while retaining much of its spectral properties ---see, e.g., \cite{CBGV09,CBGVZ10,LP09,ss08}. 
Inspired by~\cite{CBGV09,CBGVZ10}, this paper reduces the problem of node classification from graphs to trees
by extracting suitable {\em spanning trees} of the graph, which can be done quickly in many cases. 
The advantage of performing this reduction is that node prediction is much easier on trees than on
graphs. This fact has recently led to the design of very scalable algorithms with nearly optimal performance 
guarantees in the online transductive model, which comes with no statistical assumptions.
Yet, the current results in node classification on trees are not satisfactory. 
The \treeopt\ strategy of~\cite{CBGV09} is optimal to within constant factors, but only on {\em unweighted} trees.
No equivalent optimality results are available for general weighted trees.
To the best of our knowledge, the only other comparable result 
is \wta\ by~\cite{CBGVZ10}, which is optimal (within log factors) only on weighted lines. 
In fact, \wta\ can still be applied to weighted trees by exploiting an idea contained in~\cite{hlp08}. This 
is based on linearizing the tree via a depth-first visit. 
Since linearization loses most of the structural information of the tree, this approach yields suboptimal mistake bounds. 
This theoretical drawback is also confirmed by empirical performance: 
throwing away the tree structure negatively affects the practical behavior 
of the algorithm on real-world weighted graphs.

The importance of weighted graphs, as opposed to unweighted ones, is suggested by many practical
scenarios where the nodes carry more information than just labels, e.g., vectors of feature
values.
A natural way of leveraging this side information is to set the weight on the edge linking two 
nodes to be some function of the similariy between the vectors associated with these nodes.
In this work, we bridge the gap between the weighted and unweighted cases by proposing a new 
prediction strategy, called \alg, achieving a mistake bound that depends on the detailed structure 
of the weighted tree. We carry out the analysis using a notion of learning bias different from
the one used in~\cite{CBGVZ10} and more appropriate for weighted graphs. 
More precisely, we measure the regularity of the unknown node labeling via the weighted cutsize 
induced by the labeling on the tree (see Section \ref{s:lower} for a precise definition). 
This replaces the unweighted cutsize 
that was used in the analysis of \wta. When the weighted cutsize is used, a cut edge violates 
this inductive bias in proportion to its weight.
This modified bias does not prevent a fair comparison between the old algorithms and the new one: 
\alg\ specializes to \treeopt\ in the unweighted case, and to \wta\ when the input tree is a weighted line. 
By specializing \alg's analysis to the unweighted case we recover \treeopt's optimal mistake bound. 
When the input tree is a weighted line, we recover \wta's mistake bound expressed through the 
weighted cutsize instead of the unweighted one.
 The effectiveness of \alg\ on any tree is guaranteed by a
corresponding lower bound (see Section \ref{s:lower}). 

\alg\ can be viewed as a common nontrivial generalization of both \treeopt\ and \wta. Obtaining 
this generalization while retaining and extending the optimality properties of the two algorithms 
is far from being trivial from a conceptual and technical standpoint.
%
%
Since \alg\ works in the online transductive model, it can easily be applied to the more standard 
train/test (or ``batch'') transductive setting: one simply runs the algorithm on an arbitrary 
permutation of the training nodes, 
and obtains a predictive model for all test nodes. However, the implementation might take advantage of knowing 
the set of training nodes beforehand. For this reason, we present two implementations of \alg, one
for the online and one for the batch setting. Both implementations result in fast algorithms. 
In particular, the batch one is linear in $|V|$. 
This is achieved by a fast algorithm for weighted cut minimization on trees, a procedure which lies at the 
heart of \alg.

Finally, we test \alg\ against \wta, label propagation, and other competitors on real-world weighted graphs. 
In {\em almost all} cases (as expected), we report improvements over \wta\ due to the better sensitivity to the graph structure.
In some cases, we see that \alg\ even outperforms standard label propagation methods. Recall that label propagation 
has a running time per prediction which is proportional to $|E|$, where $E$ is the graph edge set. 
On the contrary, \alg\ can typically be run in {\em constant} amortized time per prediction by using Wilson's algorithm for sampling random 
spanning trees~\cite{Wil96}. By disregarding edge weights in the initial sampling phase, this algorithm
is able to draw a random (unweighted) spanning tree in time proportional 
to $|V|$ on most graphs. Our experiments reveal that using the edge weights only in the subsequent 
prediction phase causes in practice only a minor performance degradation.

\section{Preliminaries and basic notation}
\label{sec:prel}
%
Let $T = (V,E,W)$ be an undirected and weighted tree with $|V|=n$ nodes,
positive edge weights $W_{i,j} > 0$ for $(i,j) \in E$, and $W_{i,j} = 0$ for 
$(i,j) \notin E$.
A binary labeling of $T$ is any assignment $\by = (y_1,\dots,y_n) \in \spin^n$
of binary labels to its nodes. We use $(T,\by)$ to denote the resulting
labeled weighted tree.
The online learning protocol for predicting $(T,\by)$ is defined as
follows. The learner is given $T$ while $\by$ is kept hidden.
The nodes of $T$ are presented to the learner one by one, according
to an unknown and arbitrary permutation $i_1,\dots,i_n$ of $V$.
At each time step $t=1,\dots,n$ node $i_t$ is presented
and the learner must issue a prediction $\yhat_{i_t}\in\spin$ for the 
label $y_{i_t}$. Then
$y_{i_t}$ is revealed and the learner knows whether a mistake occurred.
The learner's goal is to minimize the total number of prediction mistakes.

Following previous works \cite{HPR09,hlp08,CBGV09,CBGVZ10,cgvz10}, 
we measure the regularity of a labeling $\by$ of $T$ 
in terms of $\phi$-edges, where a $\phi$-edge for $(T,\by)$
is any $(i,j) \in E$ such that $y_i \neq 
y_j$. The overall amount of irregularity in a labeled tree $(T,\by)$ 
is the \textbf{weighted cutsize} 
$\Phi^W = \sum_{(i,j) \in E^{\phi}} W_{i,j}$, where $E^{\phi} \subseteq E$ is the 
subset of $\phi$-edges in the tree. We use the weighted cutsize as our learning bias, 
that is, we want to design algorithms whose predictive performance scales with 
$\Phi^W$. Unlike the $\phi$-edge count $\Phi = |E^{\phi}|$, which is a good measure of 
regularity for unweighted graphs, the weighted cutsize takes the edge weight $W_{i,j}$ into 
account\footnote
{ 
The weight value $W_{i,j}$ typically encodes the strength
of the connection $(i,j)$. In fact, when the nodes of a graph host more 
information than just binary labels, e.g., a vector of feature velues,
then a reasonable choice is to set $W_{i,j}$ to be some (decreasing) function 
of the distance between the feature vectors sitting at the two nodes $i$ and $j$
. See also Remark \ref{r:2}.
} 
when measuring the irregularity of a $\phi$-edge $(i,j)$.
In the sequel, when we measure the distance between any pair of nodes $i$ and $j$ on the input tree $T$
we always use the resistance distance metric $d$, that is, $d(i,j) = \sum_{(r,s)\in\pi(i,j)} \tfrac{1}{W_{r,s}}$, where $\pi(i,j)$ is the unique path connecting $i$ to $j$.

\section{A lower bound for weighted trees}\label{s:lower}
%
In this section we show that the weighted cutsize can be used as a lower bound on the 
number of online mistakes made by any algorithm on any tree. In order to do so
(and unlike previous papers on this specific subject ---see, e.g., \cite{CBGVZ10}), 
we need to introduce a more refined notion of adversarial ``budget".
Given $T = (V,E,W)$, let $\xi(M)$ be the maximum number of edges of $T$ such 
that the sum of their weights does not exceed $M$,
\(
\xi(M) = \max\theset{|E'|}{E' \subseteq E,\; \sum_{(i,j) \in E'} w_{i,j} \le M}~.
\)
We have the following simple lower bound (all proofs are omitted from this extended abstract).
%
\begin{theorem}
\label{th:lb}
For any weighted tree $T = (V,E,W)$ there exists a randomized label assignment to $V$
such that any algorithm can be forced to make at least $\xi(M)/2$ online mistakes in expectation, while $\Phi^W \le M$.
\end{theorem}
Specializing~\cite[Theorem~1]{CBGVZ10} to trees gives the lower bound $K/2$ under the 
constraint $\Phi \le K \le |V|$. The main difference between the two bounds is the measure 
of label regularity being used: Whereas Theorem~\ref{th:lb} uses $\Phi^W$, which depends on 
the weights, \cite[Theorem~1]{CBGVZ10} uses the weight-independent quantity 
$\Phi$. This dependence of the 
lower bound on the edge weights is consistent with our learning bias, stating that a heavy 
$\phi$-edge violates the bias more than a light one.
Since $\xi$ is nondecreasing, the lower bound implies a number of 
mistakes of at least $\xi(\Phi^W)/2$. Note that $\xi(\Phi^W) \ge \Phi$ for 
any labeled tree $(T,\by)$. Hence, whereas a 
constraint $K$ on $\Phi$ implies forcing at least $K/2$ mistakes, a constraint $M$ on 
$\Phi^W$ allows the adversary to force a potentially larger number of mistakes.

In the next section we describe an algorithm whose mistake bound nearly matches the
above lower bound on any weighted tree when using $\xi(\Phi^W)$ as the 
measure of label regularity.

\section{The Shazoo algorithm}
\label{s:alg}
%
In this section we introduce the \alg\ algorithm, and relate it to 
previously proposed methods for online prediction on unweighted trees
(\treeopt\ from \cite{CBGV09}) and weighted line graphs (\wta\ from \cite{CBGVZ10}).
In fact, \alg\ is optimal on any weighted tree, and reduces to \treeopt\ on unweighted 
trees and to \wta\ on weighted line graphs. Since \treeopt\ and \wta\ are 
optimal on \textit{any} unweighted tree and \textit{any} weighted line graph, 
respectively, \alg\ necessarily contains elements of both of these algorithms.

In order to understand our algorithm, we now define some relevant 
structures of the input tree $T$. See Figure\ \ref{fig:hinge-trees_shazoo} (left) 
for an example.
These structures evolve over time according to the set of observed labels. 
First, we call \textbf{revealed} a 
node whose label has already been observed by the online learner; otherwise, a node is \textbf{unrevealed}.
A \textbf{fork} is any unrevealed node connected to at least three different 
revealed nodes by edge-disjoint paths. A \textbf{hinge node} is either
a revealed node or a fork. A \textbf{hinge tree} is any component of the 
forest obtained by removing from $T$ all \textsl{edges} incident to hinge 
nodes; hence any fork or labeled node forms a $1$-node hinge tree. 
When a hinge tree $H$ contains only one hinge node, a \textbf{connection node} for $H$ is 
the node contained in $H$. In all other cases, we call a connection node for $H$
any node outside $H$ which is adjacent to a node in $H$. 
A \textbf{connection fork} is a connection node which is also a fork.
Finally, a \textbf{hinge line} is any path connecting two hinge nodes such that 
no internal node is a hinge node. 
\begin{figure}[h!]
\begin{center}
\includegraphics[scale=0.45]{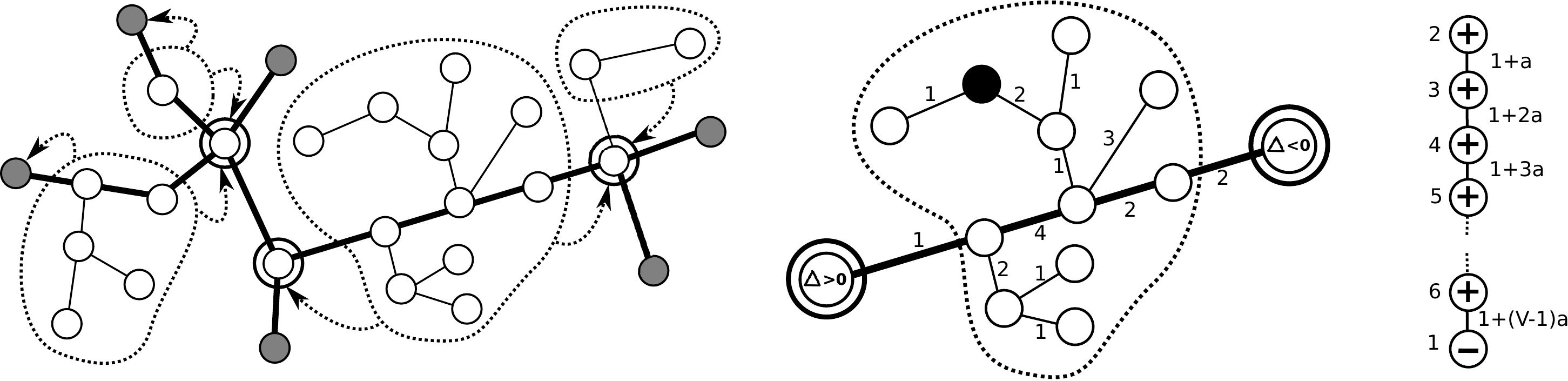}
\end{center}
\caption{\label{fig:hinge-trees_shazoo} 
\textbf{Left:} An input tree. Revealed nodes are dark grey, forks are doubly circled,
and hinge lines have thick black edges.
The hinge trees not containing hinge nodes (i.e., the ones that are not singletons) 
are enclosed by dotted lines. The dotted arrows point to the connection node(s) of such hinge trees.
\textbf{Middle:} The predictions of \alg\ on the nodes of a hinge tree. 
The numbers on the edges denote edge weights. 
At a given time $t$, \alg\ uses the value of $\Delta$
on the two hinge nodes (the doubly circled ones, which are also forks in this case), 
and is required to issue a prediction on node $i_t$
(the black node in this figure).
Since $i_t$ is between a positive $\Delta$ hinge node and a negative $\Delta$ hinge node, 
\alg\ goes with the one which is closer in resistance distance, hence predicting $\yhat_{i_t} = -1$. 
\textbf{Right:} A simple example where the mincut prediction strategy does not work well
in the weighted case. In this example,
mincut mispredicts all labels, yet $\Phi=1$, and the ratio of $\Phi^W$ to the 
total weight of all edges is about $1/|V|$. The labels to be predicted are presented according to the 
numbers on the left of each node. Edge weights are also displayed, where $a$ is a very small constant.
}
\end{figure}

Given an unrevealed node $i$ and a label value $y\in\spin$, the \textbf{cut function} 
$\cut(i,y)$ is the value of the minimum weighted cutsize of $T$ over all 
labelings $\by \in \spin^n$ consistent with the labels seen so far and such 
that $y_i=y$. 
Define $\Delta(i) = \cut(i,-1)-\cut(i,+1)$ if $i$ is unrevealed, 
and $\Delta(i) = y_i$, otherwise.
The algorithm's pseudocode is given in Algorithm \ref{alg:shazoo}.
At time $t$, in order to predict the label $y_{i_t}$ of node $i_t$, 
\alg\ calculates $\Delta(i)$ for all connection nodes $i$ of $H(i_t)$, where
$H(i_t)$ is the hinge tree containing $i_t$. Then the algorithm predicts
$y_{i_t}$ using the label of the connection node $i$ of $H(i_t)$ 
which is closest to $i_t$ and such that $\Delta(i) \neq 0$ (recall from Section~\ref{sec:prel} that all distances/lengths are measured using the resistance metric). Ties are broken arbitrarily.
If $\Delta(i) = 0$ for all connection nodes $i$ in $H(i_t)$ then 
\alg\ predicts a default value ($-1$ in the pseudocode).
%
%
%
%
%
%
%
%
%
%
If $i_t$ is a fork (which is also a hinge node), 
then $H(i_t) = \{i_t\}$. In this case, $i_t$ is a connection node of 
$H(i_t)$, and obviously the one closest to itself. 
Hence, in this case \alg\ predicts $y_t$ simply by
$\yhat_{i_t} = \sgn\bigl(\Delta(i_t)\bigr)$. See Figure \ref{fig:hinge-trees_shazoo} (middle) for an example.
%
\begin{algorithm2e}[h]
\SetKwSty{bf} 
\SetKwFor{For}{for}{}{}
\SetKwIF{If}{ElseIf}{Else}{if}{}{else if}{else}{}
\SetKwFor{While}{while}{}{}
\For{$t=1 \ldots n$}
{
Let $C\bigl(H(i_t)\bigr)$ be the set of the connection nodes 
$i$ of $H(i_t)$ for which $\Delta(i) \neq 0$\\
\eIf{$C\bigl(H(i_t)\bigr) \not\equiv \emptyset$}
{
Let $j$ be the node of $C\bigl(H(i_t)\bigr)$ closest to $i_t$\\
Set $\yhat_{i_t} = \sgn\bigl(\Delta(j)\bigr)$
}
{
\ Set $\yhat_{i_t} = -1$ (default value)
}
}
\caption{\alg}
\label{alg:shazoo}
\end{algorithm2e}
%
On unweighted trees, computing $\Delta(i)$ for a connection node $i$ 
reduces to the Fork Label Estimation Procedure in \cite[Lemma 13]{CBGV09}. 
On the other hand, predicting with the label of the connection node closest to $i_t$ in 
resistance distance is reminiscent of the nearest-neighbor prediction of \wta\ on weighted line graphs 
\cite{CBGVZ10}.
In fact, as in \wta, this enables to take advantage of labelings whose
$\phi$-edges are light weighted.
An important limitation of \wta\ is that this algorithm linearizes the input tree. 
On the one hand, this greatly simplifies the analysis of nearest-neighbor prediction; 
on the other hand, this prevents exploiting the structure of $T$, 
thereby causing logaritmic slacks in the upper bound of \wta. The \treeopt\ algorithm, instead, 
performs better when the unweighted input tree is very different from a line graph (more 
precisely, when the input tree cannot be decomposed into long edge-disjoint paths, 
e.g., a star graph). 
Indeed, \treeopt's upper bound does not suffer from logaritmic slacks, and is tight 
up to 
constant factors on any unweighted tree. Similar to \treeopt, \alg\ does not linearize the 
input tree and extends to the weighted case \treeopt's superior performance, also confirmed  
by the experimental comparison reported in Section \ref{s:exp}.

In Figure \ref{fig:hinge-trees_shazoo} (right)  
we show an example that highlights the importance of using the $\Delta$ function 
to compute the fork labels. Since $\Delta$ predicts a fork $i_t$ with the label that minimizes 
the weighted cutsize of $T$ consistent with the revealed labels, one may wonder whether 
computing $\Delta$ through mincut based on the number of $\phi$-edges 
(rather than their weighted sum) could be an effective prediction strategy.
Figure \ref{fig:hinge-trees_shazoo} (right) illustrates an example of
a simple tree where such a $\Delta$ mispredicts the labels of 
all nodes, when both $\Phi^W$ and $\Phi$ are small. 

\begin{remark}\label{r:1}
We would like to stress that \alg\ can also be used to predict the nodes of an arbitrary 
{\em graph} by first drawing a random spanning tree $T$ of the graph, and then 
predicting optimally on $T$ ---see, e.g.,~\cite{CBGV09,CBGVZ10}. The resulting mistake bound 
is simply the expected value of \alg's mistake bound over the random draw of $T$.
By using a fast spanning tree sampler~\cite{Wil96}, the involved computational overhead 
amounts to constant amortized time per node prediction on ``most'' graphs. 
\end{remark}
%
\begin{remark}\label{r:2}
In certain real-world input graphs, the presence of an edge 
linking two nodes may also carry
%
%
information about the extent to which the two nodes are {\em dissimilar}, rather than similar. 
This information can be encoded by the sign of the weight, and the resulting network is called 
a \textsl{signed graph}. The regularity measure is naturally extended to signed graphs by 
counting the weight of \textsl{frustrated edges}~(e.g.,\cite{IA10}), where $(i,j)$ is 
frustrated if $y_i y_j \neq \sgn(w_{i,j})$. Many of the existing algorithms for node 
classification \cite{ZGL03,hlp08,HPR09,CBGV09,HL09,CBGVZ10} can in principle be run on signed graphs. 
However, the computational cost may not always be preserved. For example, mincut~\cite{bc01}
is in general NP-hard when the graph is signed~\cite{MCRG03}.
Since our algorithm sparsifies the graph using trees, 
it can be run efficiently even in the signed case. 
We just need to re-define the $\Delta$ function as 
$\Delta(i) = \fcut(i,-1) - \fcut(i,+1)$, 
where $\fcut$ is the minimum total weight of frustrated edges consistent with the labels seen 
so far. The argument contained in Section~\ref{s:impl} for the positive edge weights 
(see, e.g., Eq. (\ref{eq:cut}) therein) allows us to show that also this version of $\Delta$ can be computed 
efficiently. The prediction rule has to be re-defined as well: We count the parity 
of the number $z$ of negative-weighted edges along the path connecting $i_t$ to the closest 
node $j \in C\bigl(H(i_t)\bigr)$, i.e., $\yhat_{i_t} = (-1)^z \sgn\bigl(\Delta(j)\bigr)$. 
\end{remark}
%
\begin{remark}
In~\cite{CBGV09} the authors note that \treeopt\ approximates a version space (Halving) algorithm on the set of tree labelings. Interestingly, \alg\ is also an approximation to a more general Halving algorithm for weighted trees. 
This generalized Halving gives a weight to each labeling consistent with the labels seen so far and with the sign 
of $\Delta(f)$ for each fork $f$. These weighted labelings, which depend on the weights of the $\phi$-edges 
generated by each labeling, are used for computing the predictions. One can show (details omitted due to 
space limitations) that this generalized Halving algorithm has a mistake bound within a constant factor 
of \alg's.
\end{remark}

\section{Mistake bound analysis and implementation}
\label{s:impl}
\vspace{-0.1in}
\newcommand{\outbd}[1]{\partial{#1}}
\newcommand{\oC}{\overline{C}}
\newcommand{\oc}{\overline{C}}
\newcommand{\oD}{\overline{\Delta}}
\newcommand{\scL}{\mathcal{L}}
\newcommand{\cFF}{C^{F'}}
\newcommand{\GT}{\Gamma^T}
\newcommand{\scT}{\mathcal{T}}
We now show that \alg\ is nearly optimal on every weighted tree $T$. We obtain an upper 
bound in terms of $\Phi^W$ and the structure of $T$, nearly matching the lower bound 
of Theorem~\ref{th:lb}. 
We now give some auxiliary notation that is strictly
needed for stating the mistake bound. 

Given a labeled tree $(T,\by)$, a \textbf{cluster} is any maximal subtree whose nodes have 
the same label. An \textbf{in-cluster line graph} is any line graph that is entirely 
contained in a single cluster. Finally, given a line graph $L$, we set
$R^W_L = \sum_{(i,j) \in L} \tfrac{1}{W_{i,j}}$, i.e., the (resistance) distance between its terminal nodes.
%
%
\begin{theorem}
\label{th:ub}
For any labeled and weighted tree $(T, \by)$, there exists a set $\scL_T$ of $\scO\bigl(\xi(\Phi^W)\bigr)$ edge-disjoint in-cluster line graphs such that 
the number of mistakes made by \alg\ is at most of the order of
\vspace{-0.05in}
\[
	\sum_{L \in \scL_T} \min\Bigl\{ |L|, 1 +\bigl\lfloor \log\bigl(1 + \Phi^W R^W_L\bigr) \bigr\rfloor \Bigr\}~.
\]
\end{theorem}
\vspace{-0.1in}
The above mistake bound depends on the tree structure through $\scL_T$. 
The sum contains
$\scO\bigl(\xi(\Phi^W)\bigr)$ terms, each one being 
at most logarithmic in the scale-free products $\Phi^W R_L^W$. 
The bound is governed by the same key quantity $\xi\bigl(\Phi^W\bigr)$ 
occurring in the lower bound of Theorem~\ref{th:lb}.
However, Theorem~\ref{th:ub} also shows that \alg\ can take advantage of 
trees that 
cannot be covered by 
long line graphs. For example, if the input tree $T$ is a weighted line 
graph, then
it is likely to contain long in-cluster lines. Hence, the factor multiplying 
$\xi\bigl(\Phi^W\bigr)$ may be of the order of 
$\log\bigl(1 +\Phi^W R^W_L\bigr)$.
If, instead, $T$ has constant diameter (e.g., a star graph), then the in-cluster 
lines can only contain a constant number of nodes, and the number of mistakes 
can never exceed
$\scO\bigl(\xi(\Phi^W)\bigr)$. This is a log factor improvement over \wta\ which, by
its very nature, cannot exploit the structure of the tree it operates on.\footnote
{
One might wonder whether an arbitrarily large gap between upper (Theorem \ref{th:ub}) 
and lower (Theorem \ref{th:lb}) bounds exists due to the extra factors 
depending on $\Phi^W R^W_L$. One way to get around this is 
to follow the analysis of \wta\ in~\cite{CBGVZ10}.
Specifically, we can adapt here 
the more general analysis from that paper (see Lemma 2 therein) that allows us
to drop, for any integer $K$, the resistance contribution of $K$ arbitrary non-$\phi$ 
edges of the line graphs in $\scL_T$
(thereby reducing $R_L^W$ for any $L$ containing any of these edges) at the cost of increasing the 
mistake bound by $K$. The details will be given in the full version of this paper.
} 


%
%
As for the implementation, we start by describing a method for calculating $\cut(v,y)$ for any unlabeled node $v$ and label 
value $y$.
Let $T^{v}$ be the maximal subtree of $T$ rooted at $v$, such that no internal node is revealed.
For any node $i$ of $T^{v}$, let $T^{v}_i$ be the subtree of $T^{v}$ rooted at $i$.
Let $\Phi_i^{v}(y)$ be the minimum weighted 
cutsize of $T^{v}_i$ consistent with the revealed nodes and such 
that $y_i=y$.
Since $\Delta(v) = \cut(v,-1) - \cut(v,+1) = \Phi_v^{v}(-1) - \Phi_v^{v}(+1)$,
our goal is to compute $\Phi_v^{v}(y)$.  
It is easy to see by induction that the quantity $\Phi^{v}_i(y)$ can be recursively defined 
as follows, where $C^{v}_i$ is the set of all children of $i$ in $T^{v}$, and $Y_j \equiv \{y_j\}$ if $y_j$ is 
revealed, and $Y_j \equiv \{-1,+1\}$, otherwise:\footnote{
The recursive computations contained in this section are reminiscent of the sum-product 
algorithm~\cite{kfl01}.
}
\begin{equation}\label{eq:cut}
    \Phi^{v}_i(y) = \left\{ \begin{array}{cl}
    {\displaystyle \sum_{j \in C^v_i} \min_{y' \in Y_j} \Bigl(\Phi^{v}_j(y')+ \Ind{y' \neq y} 
w_{i,j}\Bigr) } & \text{if $i$ is an internal node of $T^{v}$}
\\
    0 & \text{otherwise.}
    \end{array} \right.
\end{equation}
%
%
%
%
Now, $\Phi_v^{v}(y)$ can be computed through a simple depth-first visit of $T^{v}$. In all backtracking 
steps of this visit the algorithm uses (\ref{eq:cut}) to
compute $\Phi^{v}_i(y)$ for each node $i$, 
the values $\Phi^{v}_j(y)$ for all children $j$ of $i$ being calculated during the 
previous backtracking steps. The total running time is therefore linear in the number of nodes
of $T^v$.

Next, we describe the basic implementation of \alg\ for the on-line setting. A batch learning
implementation will be given at the end of this section. 
The online implementation is made up of three steps.

\textbf{1. Find the hinge nodes of subtree} $T^{i_t}$. 
Recall that a hinge-node is either a fork or a revealed node. 
Observe that a fork is incident to at least three nodes lying on different hinge 
lines. Hence, in this step we perform a depth-first
visit of $T^{i_t}$, marking each node lying on a hinge line. In order to accomplish this task,
it suffices to single out all forks marking each labeled node and, recursively, each parent of a 
marked node of $T^{i_t}$. At the end of this process we are able to single out 
the forks by counting the number of edges $(i,j)$ of each marked node $i$ 
such that $j$ has been marked, too. The remaining hinge nodes are the leaves of 
$T^{i_t}$ whose labels have currently been revealed. 

\textbf{2. Compute} $\sgn(\Delta(i))$ \textbf{for all connection forks of $H(i_t)$}. 
From the previous step we can easily find the connection node(s) of $H(i_t)$.
Then, we simply exploit the above-described technique for computing
the cut function, obtaining $\sgn(\Delta(i))$ for all connection forks $i$ of $H(i_t)$.

\textbf{3. Propagate the labels of the nodes of $C(H(i_t))$ (only if $i_t$ is not a fork)}. 
We perform a visit of $H(i_t)$ starting from every node $r \in C(H(i_t))$.
During these visits, we mark each node $j$ of $H(i_t)$ with
the label of $r$ computed in the previous step, together
with the length of $\pi(r,j)$, which is what we need for predicting any label
of $H(i_t)$ at the current time step.


The overall running time is dominated by the first step and the calculation of $\Delta(i)$. 
Hence the worst case running time is proportional to $\sum_{t \le |V|} |V(T^{i_t})|$. 
This quantity can be quadratic in $|V|$, though
this is rarely encountered in practice if the node presentation order is not adversarial.
For example, it is easy to show that in a line graph, 
if the node presentation order is random, then the total time is of the order of $|V| \log |V|$. 
For a star graph the total time complexity
is always linear in $|V|$, even on adversarial orders.

In many real-world scenarios, one is interested in the more standard problem of 
predicting the labels of a given subset of {\em test} nodes based on the available
labels of another subset of {\em training} nodes.  
Building on the above on-line implementation, we now derive an
implementation of \alg\ for this train/test (or ``batch learning'') 
setting. 
We first  show that computing $|\Phi_i^i(+1)|$ and $|\Phi_i^i(-1)|$ for all unlabeled nodes 
$i$ in $T$ takes $\scO(|V|)$ time. This allows us to compute $\sgn(\Delta(v))$ for all
forks $v$ in $\scO(|V|)$ time, and then use the first and the third steps of the on-line 
implementation. Overall, we show that predicting {\em all} labels in the test set
takes $\scO(|V|)$ time.

Consider tree $T^i$ as rooted at $i$. Given any unlabeled node $i$, 
we perform a visit of $T^i$ starting at $i$.
During the backtracking steps of this visit we use (\ref{eq:cut}) 
to calculate $\Phi^i_j(y)$ for each node $j$ in $T^i$ and label $y \in \{-1,+1\}$.
Observe now that for any pair $i,j$ of adjacent unlabeled nodes and any label 
$y \in \{-1,+1\}$, once we have obtained $\Phi^i_i(y)$, $\Phi^i_j(+1)$ and $\Phi^i_j(-1)$, 
we can compute $\Phi^j_i(y)$ in constant time, as
\(
\Phi^j_i(y) = \Phi^i_i(y) - \min_{y' \in \{-1,+1\}} \bigl(\Phi^i_j(y')+\Ind{y' \neq y} w_{i,j}\bigr)
\).
In fact, all children of $j$ in $T^i$ are descendants of $i$, while the children of $i$ 
in $T^i$ (but $j$) are descendants of $j$ in $T^j$. 
\alg\ computes $\Phi^i_i(y)$, we can compute in constant time $\Phi^j_i(y)$ for 
all child nodes $j$ of $i$ in $T^i$, and use this value for computing $\Phi^j_j(y)$.  
Generalizing this argument, it is easy to see that in the next phase we can compute
$\Phi_k^{k}(y)$ in constant time for all nodes $k$ of $T^i$ such that for all ancestors $u$ of 
$k$ and all $y \in \{-1,+1\}$, the values of $\Phi^u_u(y)$ have previously been computed.   

The time for computing  $\Phi_s^{s}(y)$  for all nodes $s$ of $T^i$ and any label $y$ is 
therefore linear in the time of performing
a breadth-first (or depth-first) visit of $T^i$, i.e., linear in the number of nodes of $T^i$. Since 
each labeled node with degree $d$
is part of at most $d$ trees $T^i$ for some $i$, we have that the total number of nodes of all 
distinct (edge-disjoint) trees $T^i$ across $i \in V$ is linear in $|V|$.

Finally, we need to propagate the connection node labels of each hinge 
tree as in the third step of the online implementation. Since also 
this last step takes linear time, we conclude that the total time for 
predicting all labels is linear in $|V|$.

\newcommand{\rst}{\textsc{rst}}
\newcommand{\nwrst}{\textsc{nwrst}}
\newcommand{\mst}{\textsc{mst}}
\newcommand{\norm}[1]{\left\|{#1}\right\|}

%
\section{Experiments}\label{s:exp}
%
We tested our algorithm on a number of real-world weighted graphs from different domains 
(character recognition, text categorization, bioinformatics, Web spam detection) against the following 
baselines:

\textbf{Online Majority Vote} (\omv).
This is an intuitive and fast algorithm for sequentially predicting the node labels is via a 
weighted majority vote over the labels of the adjacent nodes seen so far. Namely, \omv\ predicts $y_{i_t}$ 
through the sign of $\sum_s y_{i_s} w_{i_s,i_t}$, where $s$ ranges over $s < t$ such that $(i_s,i_t) \in E$.
Both the total time and space required by \omv\ are $\Theta(|E|)$.

\textbf{Label Propagation} (\labprop).
$\labprop$~\cite{ZGL03,BMN04,BDL06} is a batch transductive learning method computed by solving a 
system of linear equations which requires total time of the order of $|E|\times|V|$.
This relatively high computational cost should be taken into account when comparing 
\labprop\ to faster online algorithms. Recall that $\omv$ can be viewed as a fast ``online approximation'' 
to \labprop.

\textbf{Weighted Tree Algorithm} (\wta).
As explained in the introductory section, \wta\ can be viewed as a special case of \alg. 
When the input graph is not a line, \wta\ turns it into a line by first extracting a spanning tree 
of the graph, and then linearizing it. The implementation described in \cite{CBGVZ10} runs in constant 
amortized time per prediction whenever the spanning tree sampler runs in time $\Theta(|V|)$.

The Graph Perceptron algorithm~\cite{HPR09} is another readily available baseline.
This algorithm has been excluded from our comparison because it does not seem to be very 
competitive in terms of performance (see, e.g., \cite{CBGVZ10}), and is also computationally expensive.

In our experiments, we combined \alg\ and \wta\ with spanning trees generated in different ways (note that \omv\ and \labprop\ do not need to extract spanning trees from the input graph).

\textbf{Random Spanning Tree} (\rst). 
Following Ch.\ 4 of~\cite{LP09}, we draw a weighted spanning tree with probability 
proportional to the product of its edge weights. We also tested our algorithms combined 
with random spanning trees generated uniformly at random ignoring the edge weights 
(i.e., the weights were only used to compute predictions on the randomly generated tree) ---we call 
these spanning trees \nwrst\ (no-weight $\rst$). On most graphs, this procedure can be run in time linear
in the number of nodes~\cite{Wil96}. Hence, the combinations \alg+\nwrst\ and \wta+\nwrst\ run 
in $\mathcal{O}(|V|)$ time on most graphs.

\textbf{Minimum Spanning Tree} (\mst). 
This is the spanning tree minimizing the sum of the resistors on its edges.
This tree best approximates the original graph in terms of the trace norm distance of the 
corresponding Laplacian matrices.

Following~\cite{HPR09,CBGVZ10}, we also ran \alg\ and \wta\ using committees of spanning 
trees, and then aggregating predictions via a majority vote. The resulting algorithms are denoted by 
$k$*\alg\ and $k$*\wta, where $k$ is the number of spanning trees in the aggregation.
We used either $k = 7, 11$ or $k = 3, 7$, depending on the dataset size.

For our experiments, we used five datasets: RCV1, USPS, KROGAN, COMBINED, and WEBSPAM. WEBSPAM is a big 
dataset (110,900 nodes and 1,836,136 edges) of inter-host links created for the Web Spam 
Challenge 2008~\cite{YWS07}.\footnote{
We do not compare our results to those obtained within the challenge since we are 
only exploiting the graph (weighted) topology here, disregarding content features.
}
KROGAN (2,169 nodes and 6,102 edges) and COMBINED (2,871 nodes and 6,407 edges) are high-throughput protein-protein interaction networks of budding yeast taken from~\cite{PSGGK07} ---see~\cite{CBGVZ10} for a more complete description. 
Finally, USPS and RCV1 are graphs obtained from the USPS handwritten characters dataset (all ten categories) 
and the first 10,000 documents in chronological order of Reuters Corpus Vol.~1 (the four most frequent categories),
respectively.
In both cases, we used Euclidean $10$-Nearest Neighbor to create edges, each weight $w_{i,j}$ being equal to 
$e^{-\norm{x_i-x_j}^2/\sigma^2_{i,j}}$. We set $\sigma^2_{i,j} = \tfrac{1}{2}\bigl(\sigma^2_{i} + \sigma^2_{j}\bigr)$, where $\sigma^2_{i}$ is the average squared distance between $i$ and its $10$ nearest neighbours.

Following previous experimental settings \cite{CBGVZ10}, we
associate binary classification tasks with the five datasets/graphs
via a standard one-vs-all reduction.
Each error rate is obtained by averaging over ten randomly chosen training sets 
(and ten different trees in the case of \rst\ and \nwrst). 
WEBSPAM is natively a binary classification problem, and we used the same train/test split 
provided with the dataset: 3,897 training nodes and 1,993 test nodes (the remaining nodes being unlabeled).


In the below table, we show the macro-averaged classification error rates 
(percentages)  
achieved by the various algorithms on the first four datasets mentioned 
in the main text.
For each dataset we trained ten times over a random subset of 5\%, 10\% and 25\% of the total number
of nodes and tested on the remaining ones.
In boldface are the lowest error rates on each column, excluding \labprop\ which is used as a ``yardstick'' comparison.
Standard deviations averaged over the binary problems are small: most of the times less than 0.5\%.
\begin{center}
\begin{tiny}
\begin{tabular}{l|c|c|c|c|c|c|c|c|c|c|c|c}

\multicolumn{1}{r|}{Datasets}& \multicolumn{3}{|c|}{USPS}&\multicolumn{3}{|c|}{RCV1}&\multicolumn{3}{|c|}{KROGAN}&\multicolumn{3}{|c}{COMBINED}\\

Predictors  &5\%&10\%&25\%&5\%&10\%&25\%&5\%&10\%&25\%&5\%&10\%&25\%\\
\hline
$\alg$+$\rst$&3.62&2.82&2.02&21.72&18.70&15.68&18.11&17.68&17.10&17.77&17.24&17.34\\
$\alg$+$\nwrst$&3.88&3.03&2.18&21.97&19.21&15.95&18.11&18.14&17.32&17.22&17.21&17.53\\
$\alg$+$\mst$&\textbf{1.07}&\textbf{0.96}&\textbf{0.80}&17.71&14.87&11.73&17.46&16.92&16.30&16.79&16.64&17.15\\
\hline
$\wta$+$\rst$&5.34&4.23&3.02&25.53&22.66&19.05&21.82&21.05&20.08&21.76&21.38&20.26\\
$\wta$+$\nwrst$&5.74&4.45&3.26&25.50&22.70&19.24&21.90&21.28&20.18&21.58&21.42&20.64\\
$\wta$+$\mst$&1.81&1.60&1.21&21.07&17.94&13.92&21.41&20.63&19.61&21.74&21.20&20.32\\
\hline
7*$\alg$+$\rst$&1.68&1.28&0.97&16.33&13.52&11.07&15.54&15.58&15.46&15.12&15.24&15.84\\
7*$\alg$+$\nwrst$&1.89&1.38&1.06&16.49&13.98&11.37&15.61&15.62&15.50&15.02&15.12&15.80\\
\hline
7*$\wta$+$\rst$&2.10&1.56&1.14&17.44&14.74&12.15&16.75&16.64&15.88&16.42&16.09&15.72\\
7*$\wta$+$\nwrst$&2.33&1.73&1.24&17.69&15.18&12.53&16.71&16.60&16.00&16.24&16.13&15.79\\
\hline
11*$\alg$+$\rst$&1.52&1.17&0.89&\textbf{15.82}&\textbf{13.04}&\textbf{10.59}&\textbf{15.36}&15.40&\textbf{15.29}&14.91&15.06&15.61\\
11*$\alg$+$\nwrst$&1.70&1.27&0.98&15.95&13.42&10.93&15.40&\textbf{15.33}&15.32&\textbf{14.87}&\textbf{14.99}&15.67\\
\hline
11*$\wta$+$\rst$&1.84&1.36&1.01&16.40&13.95&11.42&16.20&16.15&15.53&15.90&15.58&\textbf{15.30}\\
11*$\wta$+$\nwrst$&2.04&1.51&1.12&16.70&14.28&11.68&16.22&16.05&15.50&15.74&15.57&15.33\\
\hline
$\omv$&24.79&12.34&2.10&31.65&22.35&11.79&43.13&38.75&29.84&44.72&40.86&33.24\\
\hline
\hline
$\labprop$&1.95&1.11&0.82&16.28&12.99&10.00&15.56&14.98&15.23&14.79&14.93&15.18\\
\end{tabular}
\end{tiny}
\end{center}
%
Next, we extract from the above table a specific comparison among \alg, \wta, and \labprop.
\alg\ and \wta\ use a single minimum spanning tree (the best performing tree type for both algorithms). Note that \alg\ consistently outperforms \wta.
\begin{center}
\includegraphics[width=0.9\textwidth]{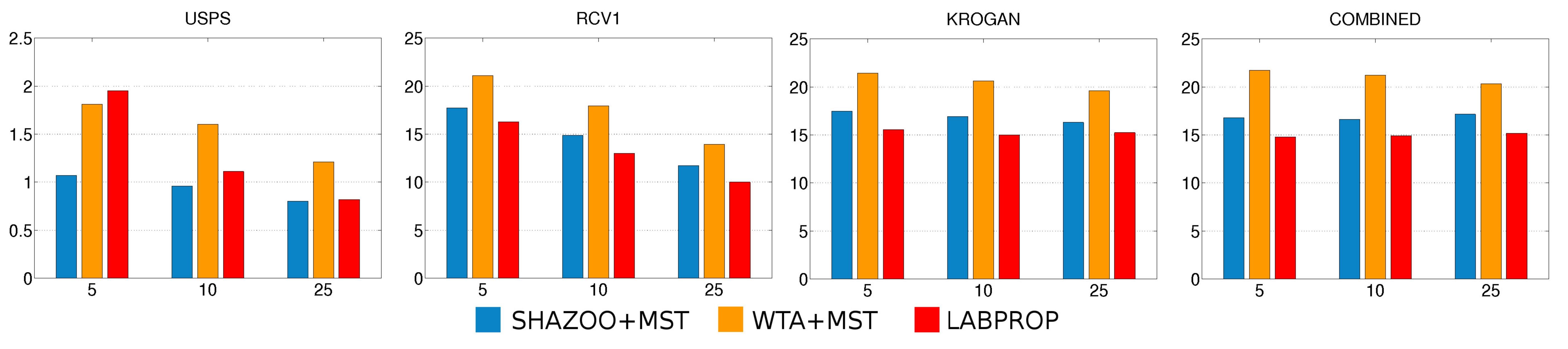}
\end{center}
%
We then report the results on WEBSPAM. \alg\ and \wta\ use only non-weighted random spanning trees (\nwrst) 
to optimize scalability. Since this dataset is extremely unbalanced (5.4\% positive labels) 
we use the average test set F-measure instead of the error rate. 

\begin{center}
\begin{tabular}{c|c|c|c|c|c|c|c}
\alg&\wta&\omv&\labprop&3*\wta&3*\alg&7*\wta&7*\alg\\
\hline
0.954&0.947&0.706&0.931&0.967&0.964&0.968&0.968
\end{tabular}
\end{center}
%

%
Our empirical results can be briefly summarized as follows:

\textbf{1.} Without using committees,
\alg\ outperforms \wta\ on all datasets, irrespective to the type 
of spanning tree being used. With committees, \alg\ works better than \wta\ almost always, 
although the gap between the two reduces.

\textbf{2.} The predictive performance of \alg+\mst\ is comparable to, and sometimes better than, that 
of \labprop, though the latter algorithm is slower. 

\textbf{3.} $k$*\alg, with $k = 11$  (or $k = 7$ on WEBSPAM) 
seems to be especially effective, outperforming \labprop, with a small (e.g., 5\%) training set size.

\textbf{4.} \nwrst\ does not offer the same theoretical guarantees as \rst, but it is extremely fast to generate
(linear in $|V|$ on most graphs --- e.g., \cite{A+08}), and in our experiments is only slightly inferior to \rst.


\appendix 

\section*{Proof of Theorem~1}
Pick any $E' \subseteq E$ such that $\xi(M) = |E'|$. Let $F$ be the forest obtained by removing from $T$ all edges in $E'$. Draw an independent random label for each of the $|E'|+1$ components of $F$ and assign it to all nodes of that component. Then any online algorithm makes in expectation at least half mistake per component, which implies that the overall number of online mistakes is $(|E'|+1)/2 > \xi(M)/2$ in expectation. On the other hand, $\Phi^W \le M$ clearly holds by construction.

\section*{Proof of Theorem~2}
%

We first give additional definitions used in the analysis, then we present the main ideas, 
and finally we provide full details.

Recall that, given a labeled tree $(T,\by)$, a \textbf{cluster} is any maximal subtree whose nodes have 
the same label. Let $\scC$ be the set of all clusters of $T$. For any cluster $C\in\scC$, let 
$M_C$ be the subset of all nodes of $C$ on which \alg\ makes a mistake. Let $\oC$ be the subtree 
of $T$ obtained by adding to $C$ all nodes that are adjacent to a node of $C$. Note that 
all edges connecting a node of $\oC \setminus C$ to a node of $C$ are $\phi$-edges.
Let $E^{\phi}_{\oC}$ be the set of $\phi$-edges in $\oC$ and let $\Phi_{\oC} = \bigl|
E^{\phi}_{\oC}\bigr|$. Let $\Phi^W_{\oC}$ be the total weight of the edges in 
$E^{\phi}_{\oC}$. Finally, recall the notation
$R^W_L = \sum_{(i,j) \in L} \tfrac{1}{W_{i,j}}$, where $L$ is any line graph.



Recall that an \textbf{in-cluster line graph} is any line graph that is entirely contained in a single 
cluster.
The main idea used in the proof below is to bound $|M_C|$ for each $C\in\scC$ in the following 
way. 
We partition $M_C$ into $\scO(|E'_{\oC}|)$ groups, where $E'_{\oC} \subseteq E_{\oC}$. Then 
we find a set $\scL_C$ of edge-disjoint in-cluster line graphs, and create a bijection between 
lines in $\scL_C$ and groups in $M_C$. We prove that the cardinality of each group is at most 
$m_L = \min\Bigl\{ |L|, 1 +\bigl\lfloor \ln\bigl(1 + \Phi^W R^W_L\bigr)\bigr\rfloor\Bigl\}$, 
where $L\in\scL_C$ is the associated line.
This shows that the subset $M_T$ of nodes in $T$ which are mispredicted by \alg\ 
satisfies
\[
    |M_T|
=
    \sum_{C\in\scC} |M_C|
\le
    \sum_{C\in\scC}\sum_{L\in\scL_C} m_L
=
    \sum_{L\in\scL_T} m_L
\]
where $\scL_T = \bigcup_{C\in\scC}\scL_C$.
Then we show that
\[
    \sum_{C \in \scC} \sum_{ (i,j) \in E'_{\oC} } w_{i,j} = \scO\bigl(\Phi^W\bigr)~.
\]
By the very definition of $\xi$, and using the bijection stated above, this implies
\[
    |\scL_T| = \sum_{C \in \scC} |\scL_C| = \scO\left(\sum_{C \in \scC} |E'_{\oC}|\right) = \scO\bigl(\xi(\Phi^W)\bigr)~,
\]
thereby resulting in the mistake bound contained in Theorem 2.

The details of the proof require further notation.
 
According to \alg\ prediction rule, 
when $i_t$ is not a fork and $C(H(i_t)) \not\equiv \emptyset$, the algorithm predicts 
$y_{i_t}$ using the label of any $j \in C\bigl(H(i_t)\bigr)$ closest to $i_t$. 
In this case, 
we call $j$ an \textbf{r-node} (reference node) for $i_t$ and the pair $\{j, (j,v)\}$, 
where
$(j,v)$ is the edge on the path between $j$ and $i_t$, an \textbf{rn-direction} (reference 
node direction). We use the shorthand notation $i^*$ to denote an r-node for $i$. 
In the special case when all connection nodes $i$ of the hinge tree containing $i_t$ have 
$\Delta(i) = 0$ (i.e., $C(H(i_t)) \equiv \emptyset$), and $i_t$ is not a fork, we call any 
closest connection node $j_0$ to $i_t$ an r-node for $i_t$ and we say that $\{j_0, 
(j_0,v)\}$ is a rn-direction for $i_t$. Clearly, we may have more than one node of $M_C$ 
associated with the same rn-direction. Given any rn-direction $\{j, (j,v)\}$, we call 
\textbf{r-line} (reference line) the line graph whose terminal nodes are $j$ and the first 
(in chronological order) node $j_0 \in V$ for which $\{j, (j,v)\}$ is a rn-direction, where 
$(j,v)$ lies on the path between $j_0$ and $j$.\footnote{
We may also have $v \equiv j_0$.
}
We denote such an r-line by $L(j,v)$.

In the special case where $j \in C$ and $j_0 \notin C$ we say that the r-line is associated with 
the $\phi$-edge of $E^{\phi}_{\oC}$
included in the line-graph. In this case we denote such an r-line by $L(u,q)$, where $(u,q) \in E^{\phi}_{\oC}$.
Figure\ \ref{fig:rn-direction} gives a pictorial example of the above concepts.

\begin{figure}[h!]
\begin{center}
\includegraphics[scale=0.60]{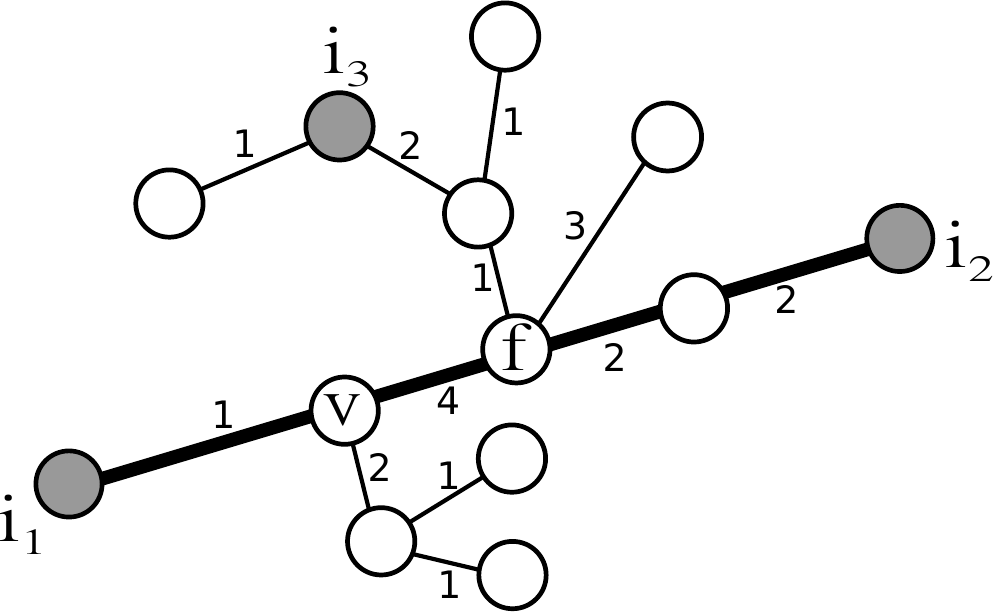}
\end{center}
\caption{\label{fig:rn-direction} 
We illustrate an example of r-node, rn-direction and r-line. The numbers near the edge lines denote edge weights. 
In order to predict $y_{i_2}$, \alg\ uses the r-node $i_1$ and the rn-direction $\{i_1, (i_1,v)\}$. 
After observing $y_{i_2}$, the hinge line connecting $i_1$ with $i_2$ (the thick black line) is created, 
which is also an r-line, since at the beginning of step $t=2$ the algorithm used $\{i_1, (i_1,v)\}$.  
In order to predict $y_{i_3}$, we still use the r-node $i_1$ and the rn-direction $\{i_1, (i_1,v)\}$. After the revelation of $y_{i_3}$, node $f$ becomes a fork.
}
\end{figure}

We now cover $M_C$ (the subset of all nodes of $C\in\scC$ on which \alg\ makes a mistake) by the following subsets:
\begin{itemize}
\item $\cf$ is the set of all forks in $M_C$.
\item $\ci$ is the subset of $M_C$ containing the nodes $i$ whose reference node 
$i^*$ belongs to $C$ (if $i$ is a fork, then $i^* = i$). Note that this set may have a 
nonempty intersection with the previous one.
\item $\co$ is the subset of $M_C$ containing the nodes $i$ such that $i^*$ does not 
belong to $C$.
\end{itemize}
Two other structures that are relevant to the proof:
\begin{itemize}
\item $\cF$ is the subset of all forks $f \in V_C$ such that $\Delta(f) \le 0$ at some step 
$t$. Since we assume the cluster label is $+1$ (see below), 
and since a fork $i_t\in V_C$ is mistaken 
only if $\Delta(i_t) \le 0$, we have $\cf \subseteq \cF$.
\item $\cFF$ is the subset of all nodes in $M_C$ that, when revealed, create a fork that 
belongs to $\cF$. 
Since at each time step at most one new fork can be created,\footnote{
In step $t$ a new fork $j$ is created when the number of edge-disjoint
paths connecting $j$ to the labeled nodes increases. 
This event occurs only when a new hinge line
$\pi(i_t, f)$ is created. When this happens, the only node for which the number 
of edge-disjoint paths connecting it to labeled nodes gets increased is the 
terminal node $j$ of the newly created hinge line.
} 
we have $|\cFF| \le |\cF|$.
\end{itemize}
%

The proof of the theorem relies on the following sequence 
of lemmas that show how to bound 
the number of mistakes made on a given cluster $C = (V_C,E_C)$. 
A major source of technical difficulties, that makes this 
analysis different and more complex than those of \treeopt\ and \wta, is
that on a weighted tree the value of $\Delta(i)$ on forks $i$ can 
potentially change after each prediction.

Without loss of generality, from now on we assume all nodes in $C$ are labeled $+1$.
Keeping this assumption in mind is crucial to understand the arguments that follow.

For any node $i \in V_C$, let $\oD(i)$ be the value of 
$\Delta(i)$ when all nodes in $\oC\setminus C$ are revealed.
\begin{lemma}\label{lm:delta}
For any fork $f$ of $C$ and any step $t=1,\dots,n$, we have $\oD(f) \le \Delta(f)$.
\end{lemma}
\begin{proof}
For the sake of contradiction, assume $\oD(f) > \Delta(f)$.
Let $T^{f}$ be the maximal subtree of $T$ rooted at $f$ such that no internal node 
of $T^{f}$ is revealed.
Now, consider the cut given by the edges of $E^{\phi}_C$  belonging to the hinge lines of $T^{f}$. This cut separates $f$ from any revealed node labeled with $-1$. The size of this cut cannot be larger 
than $\Phi^W_{\oC}$. By definition of $\Delta(\cdot)$, this implies   
$\Delta(f) \le \Phi^W_{\oC}$. However, also $\oD(f)$ cannot be larger 
than $\Phi^W_{\oC}$. Because
\[
    \oD(i_t) \le \sum_{(i,j) \in E^{\phi}_{\oC}} W_{i,j} = \Phi^W_{\oC}
\]
must hold independent of the set of nodes in $V_C$ that are revealed before time $t$,
this entails a contradiction.
\end{proof}


Let now $\xi_{\oc}$ be the restriction of $\xi$ on the subtree $\oc$, and let
$D_C$ be the set of all distinct rn-directions which the nodes of $\ci$ can be associated with. 
The next lemmas are aimed at bounding $|\cF|$ and $|D_C|$. 
We first need to introduce the superset $D_C'$ of $D_C$. 
Then, we show that for any $C$ both $|D_C'|$ and $|\cF|$ are linear in 
$\xi_{\oc}(\Phi^W_{\oC})$. 


In order to do so, we need to take into account the fact that the sign 
of $\Delta$ for the forks in the cluster can change many times during the 
prediction process. This can be done via Lemma~\ref{lm:delta}, which shows that when 
all labels in $\oC \setminus C$ are revealed then, for all fork $f\in C$,
the value $\Delta(f)$ does not increase. 
Thus, we get the largest set $D_C$ when we assume that the nodes in $\oC \setminus C$ 
are revealed before the nodes of $C$.

Given any cluster $C$, let $\sigma_{\oC}$ be the order
in which the nodes of $\oC$ are revealed. Let also $\sigma_{\oC}'$ be the permutation
in which all nodes in $C$ are revealed in the same order as $\sigma_{\oC}$, and all nodes in
$\oC \setminus C$ are revealed at the beginning, in any order.
Now, given any node revelation order $\sigma_{\oC}$, $D_C'$ can be 
defined by describing the three types of steps involved in its incremental 
construction supposing $\sigma_{\oC}'$ was the actual node revelation order.
\begin{enumerate}
\item After the first $|\oC \setminus C| = \Phi_{\oC}$ steps, $D_C'$ contains all node-edge pairs 
$\{i, (i,j)\}$ such that $i$ is a fork and $(i,j)$ is an edge laying on a hinge 
line of $\oC$. Recall that no node in $C$ is revealed yet.
\item For each step $t > 0$ when a new fork $f$ is created such that 
$\Delta(f) \le 0$ just after the revelation of $y_{i_t}$, 
we add to $D_C'$ the three node-edge pairs $\{f, (f,j)\}$, where the $(f,j)$ are the 
edges contained in the three hinge lines terminating at $f$. 
\item Let $s$ be any step where: 
(i) A new hinge line $\pi(i_s, i_s^*)$ is created,
(ii) node $i_s^*$ is a fork, and 
(iii) $\Delta(i_s^*) \le 0$ at time $s-1$. 
On each such step we add $\{i_s^*, (i_s^*,j)\}$ to $D_C'$, 
for $j$ in $\pi(i_s, i_s^*)$.
\end{enumerate}
It is easy to verify that, given any ordering $\sigma_{\oC}$ for the node 
revelation in $\oC$, we have $D_C \subseteq D_C'$. In fact, given an rn-direction
$\{i, (i,j)\} \in D_C$, if $(i,j)$ lies along one of the hinge lines 
that are present at time $0$ according to $\sigma'_{\oC}$, 
then $\{i, (i,j)\}$ must be included in $D_C'$ during
one of the steps of type 2 above, 
otherwise $\{i, (i,j)\}$ will be included in $D_C'$ during one of
the steps of type 2 or type 3.

As announced, the following lemmas show that $|D_C'|$ and $|\cF|$ 
are both of the order of $\xi_{\oc}(\Phi^W_{\oC})$.
\begin{lemma}\label{dc_first}
(i) The total number of forks at time $t = \Phi_{\oc}$ is $\scO\bigl(\xi(\Phi^W_{\oC})\bigr)$. 
(ii) The total number of elements added to $D_C'$ in the first step of its 
construction is $\scO\bigl(\xi(\Phi^W_{\oC})\bigr)$.
\end{lemma}
\begin{proof}
Assume nodes are revealed according to $\sigma_{\oC}'$.
Let $C'$ be the subtree of $\oC$ made up of all nodes
in $\oC$ that are included in any path connecting two nodes of $\oC\setminus C$. 
By their very definition, the forks at time $t=\Phi_{\oc}$ are the nodes of 
$V_{C'}$ having degree larger than two in subtree $C'$. Consider $C'$ 
as rooted at an arbitrary node of $\oC\setminus C$. The number of the leaves 
of $C'$ is equal to $|\oC\setminus C|-1$. This is in turn
$\scO\bigl(\xi_{\oC}(\Phi^W_{\oC}\bigr)$ because 
\[
\sum_{(i,j) \in E^{\phi}_{\oC}} w_{i,j} = \scO\bigl(\xi_{\oC}(\Phi^W_{\oC})\bigr)~.
\]
Now, in any tree, the sum of the degrees of nodes having degree larger 
than two cannot is at most linear in the number of leaves. Hence, at time $t = \Phi_{\oc}$ 
both the number of forks in $C$ and the cardinality of $D_C'$ are 
$\scO\bigl(\xi_{\oC}(\Phi^W_{\oC})\bigr)$. 
\end{proof}

Let now $\GT_t$ be the minimal cutsize of $T$ consistent with the labels 
seen before step $t+1$, and notice that $\GT_t$ is nondecreasing with $t$. 
%
\begin{lemma}\label{incr_cutsize}
Let $t$ be a step when a new hinge line $\pi(i_t, q)$ is created 
such that $i_t, q \in V_C$. If just after step $t$ we have $\Delta(q) \le 0$, 
then  $\GT_t-\GT_{t-1} \ge w_{u,v}$, where $(u,v)$ is the lightest 
edge on $\pi(i_t, q)$.
\end{lemma}
\begin{proof}
Since $\Delta(q) \le 0$ and $\pi(i_t, q)$ is completely included in
$C$,  we must have $\Delta(q) \le 0$ just before the revelation of $y_{i_t}$. 
This implies that the difference $\GT_t-\GT_{t-1}$ cannot be smaller
than the minimum cutsize that would be created on $\pi(i_t, q)$ by 
assigning label $-1$ to node $q$.
\end{proof}
\begin{lemma}\label{dc_second}
Assume nodes are revealed according to $\sigma_{\oC}'$.
Then the cardinality of $\cF$ and the total number of elements 
added to $D_C'$ during the steps of type 2 above are both linear in 
$\xi_{\oc}(\Phi^W_{\oC})$.
\end{lemma}
\begin{proof}
Let  $\cF_0$ be the set of forks in $V_C$ such that 
$\oD(f) \le 0$ at some time $t \le |V|$. Recall that, by
definition, for each fork $f \in \cF$ there exists a step $t_f$ 
such that $\Delta(f) \le 0$. Hence, Lemma \ref{lm:delta} implies that, 
at the same step $t_f$, for each fork $f \in \cF$ we have $\oD(f) \le 0$. 
Since $\cF$ is included in $\cF_0$, we can bound $|\cF|$ by $|\cF_0|$, 
i.e., by the number of forks $i \in V_C$ such that $\Delta(i) \le 0$,
under the assumption that $\sigma_{\oC}'$ is the actual revelation order 
for the nodes in $\oC$.

Now, $|\cF_0|$ is bounded by the number of forks created in the first $|\oC\setminus C| = \Phi_{\oC}$ steps, which is equal to $\scO\bigl(\xi(\Phi^W_{\oC})\bigr)$ plus the number of forks $f$ created at some later step and such that $\Delta(f) \le 0$ right after their creation. 
Since nodes in $\oC$ are revealed according to $\sigma_{\oC}'$, the condition
$\Delta(f) >0$ just after the creation of a fork $f$ implies that we will never 
have $\Delta(f) \le 0$ in later stages. Hence this  fork $f$ belongs neither to 
$\cF_0$ nor to $\cF$.

In order to conclude the proof, it suffices to bound from above 
the number of elements added to $D_C'$ in the steps of type 2 above. 
From Lemma \ref{incr_cutsize}, we can see that for each fork $f$ created 
at time $t$ such that $\Delta(f) \le 0$ just after the revelation of node $i_t$, 
we must have 
$|\GT_t-\GT_{t-1}| \ge w_{u,v}$, where $(u,v)$ is the lightest edge 
in $\pi(i_t, f)$. Hence, we can injectively
associate each element of $\cF$ with an edge of $E_C$, in such a way that 
the sum of the weights of these edges is bounded by $\Phi^W_{\oC}$. 
By definition of $\xi$, we can therefore conclude that the total number of 
elements added to $D_C'$ in the steps of type 2 is $\scO\bigl(\xi(\Phi^W_{\oC})\bigr)$.
\end{proof}

With the following lemma we bound the number of nodes of $\ci \setminus \cFF$ associated with every 
rn-direction and show that one can perform a transformation of the r-lines so as to make them edge-disjoint. This transformation is crucial for finding the set $\scL_T$ appearing in the theorem statement.
Observe that, by definition of r-line,
we cannot have two r-lines such that each of them includes only one terminal node of the other.
Thus, let now $F_C$ be the forest where each node is associated with an r-line and where the parent-child
relationship expresses that 
(i) the parent r-line contains a terminal node of the child r-line, together with 
(ii) the parent r-line and the child r-line are not edge-disjoint. $F_C$ is, in fact, a forest of r-lines.
We now use $m_{L(j,v)}$
%
for bounding the number of mistakes associated with a given rn-direction $\{i, (j,v)\}$ or with a given $\phi$-edge
$(j,v)$.
Given any connected component $T'$ of $F_C$, let finally $m_{T'}$ be the total  number of nodes of $\ci \setminus \cFF$ associated with the rn-directions $\{i,(i,j)\}$ of all r-lines $L(i,j)$ of $T'$.

\begin{lemma}
\label{lm:m_L}
Let $C$ be any cluster. Then:
\begin{itemize}
\item[(i)] The number of nodes in $\ci \setminus \cFF$ associated with a given 
rn-direction $\{j,(j,v)\}$ is of the order of 
$m_{L(i,j)}$.
\item[(ii)] The number of nodes in $\co \setminus \cFF$ associated with a given 
$\phi$-edge $(u,q)$ is of the order of 
$m_{L(u,q)}$.

\item[(iii)] Let  $L(j_r,v_r)$ be the r-line associated with the root of any connected component $T'$ of $F_C$.
$m_{T'}$  must be at most of the same order of
\[
    \sum_{L(j,v) \in \scL(L(j_r,v_r))} m_{L(j,v)} + |V_{T'}|
\]
where $ \scL(L(j_r,v_r))$ is a set of $|V_{T'}|$ edge-disjoint line graphs completely contained in $L(j_r,v_r)$.  
\end{itemize}
\end{lemma}
\begin{proof}
We will prove only (i) and (iii), (ii) being similar to (i). 
Let $i_t$ be a node in $\ci \setminus \cFF$ associated with a given 
rn-direction $\{j,(j,v)\}$.
There are two possibilities: 
(a) $i_t$ is in $L(j,v)$ or 
(b) the revelation of $y_{i_t}$ creates a fork $f$ in $L(j,v)$ 
such that $\Delta(f) > 0$ for all steps $s \ge t$. 
Let now $i_{t'}$ be the next node (in chronological order) of 
$\ci \setminus \cFF$ associated with $\{j,(j,v)\}$.
The length of $\pi(i_{t'}, i_t)$ cannot be smaller than the length of 
$\pi(i_{t'}, j)$ (under condition (a)) or smaller than the length of $\pi(f, j)$ 
(under condition (b)).

This clearly entails a dichotomic behaviour in the sequence
of mistaken nodes in $\ci \setminus \cFF$ associated with $\{j,(j,v)\}$.
Let now $p$ be the node in $L(j,v)$ which is
farthest from $j$ such that the length of $\pi(p,j)$ is not larger than $\Phi^W$. 
Once a node in $\pi(p,j)$ is revealed or becomes a fork $f$ satisfying 
$\Delta(f) > 0$ for all steps $s \ge t$, we have $\Delta(j) > 0$ 
for all subsequent steps (otherwise, this would contradict 
the fact that the total cutsize of $T$ is $\Phi^W$).
Combined with the above sequential dichotomic behavior, this shows that
the number of nodes of $\ci \setminus \cFF$ associated with a given
rn-direction $\{j,(j,v)\}$ can be at most of the order of
\[
\min\left\{ |L(j,v)|,\, 1 +\left\lfloor \log_2 \left(\frac{R^W_{L(j,v)}+(\Phi^W)^{-1}}{(\Phi^W)^{-1}}\right) \right\rfloor \right\} 
= 
m_{L(j,v)}~.
\]
Part (iii) of the statement can be now proved in the following way. Suppose now that an r-line $L(j, v)$, having  $j$ and $j_0$ as terminal nodes, 
includes the terminal node $j'$ of another r-line $L(j',v')$, having $j'$ and $j_0'$ as terminal nodes. Assume also that the two r-lines are not edge-disjoint. 
If $L(j',v')$ is partially included in $L(j, v)$, i.e., if $j_0'$ does not belong to $L(j, v)$,
then $L(j',v')$ can be broken into two sub-lines: the first one has 
$j'$ and $k$ as terminal nodes, being $k$ the node in $L(j, v)$ which is 
farthest from $j'$; the second one has $k$ and $j_0'$ as terminal nodes. 
It is easy to see that $L(j, v)$ must be created before 
$L(j',v')$ and $j_0$ is the only node of the second sub-line
that can be associated with 
the rn-direction $\{j', (j',v')\}$.
This observation reduces the problem to considering that in $T'$
each r-line that is not a root is completely included in its parent. 

Given an r-line $L(u,q)$ having $u$ and $z$ as terminals, we denote by $m_{\pi(u,z)}$
the quantity $m_{L(u,q)}$.
 
Consider now the simplest case in which $T'$ is formed by only
two r-lines: the parent r-line $L(j_p,v_p)$, which completely contains the child r-line $L(j_c,v_c)$.
Let $s$ be the step in which the first node $u$ of $L(j_p,v_p)$ becomes a hinge node.
 After step $s$, $L(j_p,v_p)$ can be vieved as broken in two edge-disjoint sublines
 having $\{j_p, u\}$ and $\{j_0, u\}$ as terminal node sets, where $j_0$ is one of the terminal of $L(j_p,v_p)$. Thus,
\[
	m_{T'}
\le
	\max_{u \in V_{L(j_p,v_p)}} m_{\pi(j_p,u)} + m_{\pi(u,j_0)} + 1~.
\]
 Generalizing this argument for every component $T'$ of $F_C$, and using the above observation
 about the partially included r-lines, we can state that, for any component $T'$ of $F_C$, $m_{T'}$ is of 
the order of 
 \[
 \max_{u_1, \ldots, u_{N} \in V_{L(j_p,v_p)}} \Bigl(m_{\pi(j_p,u_1)} + 
 m_{\pi(u_{N},j_0)} +  \sum_{k = 1}^{N-1} m_{\pi(u_{k},u_{k+1})} + 2|V_{T'}|\Bigr)
\]
where $N = |V_{T'}|-1$.
This entails that we can define $\scL(L(j_r,v_r))$  as the union of $\{\pi(j_p,u_1), \pi(u_{N},j_0)\}$ and
$\bigcup_{k = 1}^{N-1} \pi(u_{k},u_{k+1})$, which concludes the proof.

\end{proof}
\begin{lemma}\label{dc_third}
The total number of elements added to $D_C'$ during steps of type 3 above is
of the order of $\xi_{\oc}(\Phi^W_{\oC})$.
\end{lemma}
\begin{proof}
Assume nodes are revealed according to $\sigma_{\oC}'$, and
let $s$ be any type-3 step when a new element is added to $D_C'$.
There are two cases:
(a) $\Delta(i_s^*) \le 0$ at time $s$ or 
(b) $\Delta(i_s^*) > 0$ at time $s$.

Case (a). Lemma \ref{incr_cutsize} combined with the fact that all hinge-lines 
created are edge-disjoint, ensures that we can injectively associate each of 
these added elements 
with an edge of $E_C$ in such a way that the total weight of these edges is bounded 
by $\Phi^W_{\oC}$. This in turn implies that the total number of elements added to 
$E_C$ is $\scO\bigl(\xi_{\oc}(\Phi^W_{\oC})\bigr)$.

Case (b). Since we assumed that nodes are revealed according to $\sigma_{\oC}'$, we have that
$\Delta(i_s^*)$ is positive for all steps $t > s$. Hence we have that case (b) can occur
only once for each of such forks $i_s^*$. Since this kind of fork belongs to 
$\cF$, we can use Lemma~\ref{dc_second} and conclude
that (b) can occur at most $|\cF|= \scO\bigl(\xi_{\oc}(\Phi^W_{\oC})\bigr)$ times.
\end{proof}
\begin{lemma}
\label{rn-direction_cf}
\label{lm:number_rnd}
With the notation introduced so far, we have $|D_C| = \scO\bigl(\xi_{\oc}(\Phi^W_{\oC})\bigr)$.
\end{lemma}
\begin{proof}
Combining Lemma \ref{dc_first}, Lemma \ref{dc_second}, and Lemma \ref{dc_third} 
we immediately have $D_C' = \scO\bigl(\xi_{\oc}(\Phi^W_{\oC})\bigr)$. 
The claim then follows from $D_C \subseteq D_C'$.
\end{proof}

We are now ready to prove the theorem.

\begin{proofof}{Theorem~2}
Let $F_T$ be the union of $F_C$ over $C \in \scC$.
Using Lemma~\ref{lm:number_rnd} we deduce $|V_{F_C}| = \Phi_{\oC} + \scO\bigl(\xi_{\oc}(\Phi^W_{\oC})\bigr) = \scO\bigl(\xi_{\oc}(\Phi^W_{\oC})\bigr)$, where the term $\Phi_{\oC}$ takes into account that at most one r-line of $F_C$ may be associated with each $\phi$-edge of $\oC$.

By definition of $\xi(\cdot)$, this implies $|V_{F_T}| = \scO\bigl(\xi(\Phi^W)\bigr)$.
Using part (i) and (ii) of Lemma~\ref{lm:m_L} we have
$|M_T| \le |\cf| + |\ci| + |\co| \le |\cF| + |\cFF| + \sum_{L \in V_{F_T}} m_L \le \sum_{L \in V_{F_T}} m_L + \scO\bigl(\xi(\Phi^W)\bigr)$.

Let now $\scT(F_T)$ be the set of components of $F_T$. Given any tree $T' \in \scT(F_T)$, let $r(T')$ be the r-line root of $T'$.
Recall that, by part (iii) of Lemma~\ref{lm:m_L} for any tree $T' \in \scT(F_T)$ we can find a set $\scL(r(T'))$ of $|V_{T'}|$ edge-disjoint line graphs all included in $r(T')$ such that
$m_{T'}$ is of the order of $\sum_{L \in \scL_{T'}(r(T'))} m_L + |V_{T'}|$.
Let now $\scL_T'$ be equal to $\cup_{T' \in \scT(F_T)} \scL(r(T'))$.
Thus we have
\[
    |M_T| = \scO\left(\sum_{L \in \scL_T'} m_L+ |V_{F_T}| + \xi(\Phi^W)\right) = 
\scO\left(\sum_{L \in \scL_T'} m_L + \xi(\Phi^W)\right)~.
\]
Observe that $\scL_T'$ is not an edge disjoint set of line graphs included in $T$ only because each $\phi$-edge may belong to two different lines of $\scL_T'$.
By definition of $m_L$, for any line graphs $L$ and $L'$, where $L'$ is obtained from $L$ by removing one of the two terminal nodes and the edge incident to it, we have $m_{L'} = m_L + \scO(1)$.
If, for each $\phi$-edge shared by two line graphs of $\scL_T'$, we shorten the two line graphs so as 
no one of them includes the $\phi$-edge, we obtain a new set of edge-disjoint line graphs $\scL_T$ such that  $\sum_{L \in \scL_T'} m_L = \sum_{L' \in \scL_T} + \xi(\Phi^W)$. Hence, we finally obtain  $|M_T| = \scO\Bigl(\sum_{L' \in \scL_T} m_{L'}+ \xi(\Phi^W)\Bigr) = \scO\Bigl(\sum_{L' \in \scL_T} m_{L'}\Bigr)$, where in the last equality we used the fact that $m_{L'} \ge 1$ for all line graphs $L'$.
\end{proofof}

\end{document}